\documentclass[1p]{elsarticle}
\makeatletter
\def\ps@pprintTitle{%
 \let\@oddhead\@empty
 \let\@evenhead\@empty
 \def\@oddfoot{\centerline{\thepage}}%
 \let\@evenfoot\@oddfoot}
\makeatother

\usepackage{lineno,hyperref}

\usepackage{url}            
\usepackage{booktabs}       
\usepackage{amsfonts}       
\usepackage{nicefrac}       
\usepackage{microtype}      
\usepackage{graphicx}
\usepackage{dblfloatfix}
\usepackage{bm}

\usepackage{amsmath, amsthm, amssymb, amsfonts, cancel}
\usepackage{ifsym, subfig, wrapfig}
\usepackage{enumitem}
\usepackage{float} 
\usepackage{algorithm}
\usepackage{algorithmic}

\def\Var{{\rm Var}}

\def\bp{ \textbf{p} }

\def\b1{ {\bm{1}} }

\def\balpha{ \bm{\alpha} }

\def\btheta{ \bm{\theta} }
\def\nn{{ \parallel   }}

\def\EE{{ \mathbb{E}  }}

\def\diag{{ \text{diag}   }}

\def\be{{ \mathbf{e}  }}
\def\bx{{ \mathbf{x}  }}
\def\by{{ \mathbf{y}  }}

\newcommand{\defequal}{ \stackrel{\rm def}{=}  }

\newtheorem{theorem}{Theorem}

\newtheorem{lemma}{Lemma}

\modulolinenumbers[10]

\journal{Neural Networks}

\begin{document}

\begin{frontmatter}

\title{Information Aware Max-Norm Dirichlet Networks for Predictive Uncertainty Estimation}


\author{Theodoros Tsiligkaridis\corref{mycorrespondingauthor}}
\address{MIT Lincoln Laboratory \\ Lexington, MA 02421}
\cortext[mycorrespondingauthor]{Corresponding author}
\ead{ttsili@ll.mit.edu}




\begin{abstract}
Precise estimation of uncertainty in predictions for AI systems is a critical factor in ensuring trust and safety. Deep neural networks trained with a conventional method are prone to over-confident predictions. In contrast to Bayesian neural networks that learn approximate distributions on weights to infer prediction confidence, we propose a novel method, Information Aware Dirichlet networks, that learn an explicit Dirichlet prior distribution on predictive distributions by minimizing a bound on the expected max norm of the prediction error and penalizing information associated with incorrect outcomes. Properties of the new cost function are derived to indicate how improved uncertainty estimation is achieved. Experiments using real datasets show that our technique outperforms, by a large margin, state-of-the-art neural networks for estimating within-distribution and out-of-distribution uncertainty, and detecting adversarial examples.
\end{abstract}


\end{frontmatter}

\linenumbers

\section{Introduction} \label{intro}
Deep learning systems have achieved state-of-the-art performance in various domains \cite{LeCun:2015}. The first successful applications of deep learning include large-scale object recognition \cite{Krizhevsky:NIPS:2012} and machine translation \cite{Sutskever:2014, Wu:2016}. While further advances have achieved strong performance and often surpass human-level ability in computer vision \cite{Geirhos:NIPS:2018, He:2015, Ciresan:2012}, speech recognition \cite{Xiong:2017, Hinton:2012}, medicine \cite{Wang:2016}, bioinformatics \cite{Alipanahi:2015}, other aspects of deep learning are less well understood. Conventional neural networks (NNs) are overconfident in their predictions \cite{Guo:ICML:2017} and provide inaccurate predictive uncertainty \cite{Louizos:ICML:2017}. NNs have to be accurate, but also provide an indicator of when an error is likely to be made. Intepretability, robustness, and safety are becoming increasingly important as deep learning is deployed across various industries including healthcare, autonomous driving and cybersecurity.

Uncertainty modeling in deep learning is a crucial aspect that has been the topic of various Bayesian neural network (BNN) research studies \cite{Blundell:ICML:2015, Kingma:NIPS:2015, Gal:ICML:2016, Molchanov:ICML:2017}. BNNs capture parameter uncertainty of the network by learning distributions on weights and estimate a posterior predictive distribution by approximate integration over these parameters. The non-linearities embedded in deep neural networks make the weight posterior intractable and several tractable approximations have been proposed and trained using variational inference \cite{Blundell:ICML:2015, Kingma:NIPS:2015, Molchanov:ICML:2017, Gal:ICML:2016, Li:ICML:2017}, the Laplace approximation \cite{MacKay:1992, Ritter:2018}, expectation propagation \cite{Hernandez:2015, Sun:2017}, and Hamiltonian Monte Carlo \cite{Chen:2014}. The success of approximate BNN methods depends on how well the approximate weight distributions match their true counterparts, and their computational complexity is determined by the degree of approximation. Most BNNs take more effort to implement and are harder to train in comparison to conventional NNs. Furthermore, approximate integration over the parameter uncertainties increases the test time due to posterior sampling, and yields an approximate predictive distribution using stochastic averaging. Thus, it is of interest to develop methods that provide good uncertainty estimates while reusing the training pipeline and maintaining scalability. To this end, a simple approach was proposed that combines NN ensembles with adversarial training to improve predictive uncertainty estimates in a non-Bayesian manner \cite{Lak:2017}, but is computationally expensive. It is also known that deterministic NNs are brittle to adversarial attacks \cite{Goodfellow:2014, Kurakin:2017}. Predictive uncertainty can be used to reason about neural network predictions and detect when a network is likely to make an error, identify anomalous examples, and detect adversarial attacks. Recent works \cite{Sensoy:NIPS:2018, Malinin:2019} explicitly use the Dirichlet distribution to model distributions of class compositions and propose to learn its parameters by training deterministic neural networks. This approach yields closed-form predictive distributions and outperforms BNNs in uncertainty quantification for out-of-distribution (OOD) and adversarial queries. However, these methods require OOD data during training which is an unrealistic assumption, and uncertainty estimation performance can be improved for out-of-distribution and adversarial examples. Furthermore, uncertainty quantification for within-distribution queries were not studied.

In this paper, we propose Information Aware Dirichlet (IAD) networks that deliver accurate predictive uncertainty by learning distributions on class probability vectors through minimizing a regularized $L_{\infty}$ norm of the prediction error under a Dirichlet prior distribution. Our method improves upon the mean-square error loss used in \cite{Sensoy:NIPS:2018} because the $L_2$ norm is sensitive to outlier scores while the $L_\infty$ norm minimizes the cost of the highest prediction error among the classes which further tends to yield higher uncertainties for misclassifications because the effect of favoring one class over others is mitigated. In addition, we improve upon the KL-loss in \cite{Malinin:2019} because the proposed loss does not enforce the hard constraint that all examples yield sharp Dirichlet priors and do not rely on OOD data at training time.

Our contributions include the following:
\begin{itemize}
	\item First, a new training loss is proposed based on minimizing an approximation to the expected $L_{\infty}$ norm of the prediction error under a Dirichlet prior distribution. A closed-form approximation to this loss is derived.
	\item Second, a regularization loss is proposed to align the concentration parameters to an information direction that minimizes information captured associated with incorrect outcomes.
	\item Third, an analysis is provided that shows how properties of the new loss function improve uncertainty estimation.
	\item Finally, we demonstrate on real datasets where our technique improves upon uncertainty quantification for within-distribution, out-of-distribution and adversarial examples.
\end{itemize}

\section{Dirichlet Prior Networks} \label{sec:dpn}

Outputs of neural networks for classification tasks are probability vectors over classes. The basis of our approach is an explicit prior on class probability distributions \cite{Malinin:2019}. Given dataset $\mathcal{D}=\{(\bx_i,\by_i)\}$, the class probability vectors for sample $i$ given by $\bp_i$ are modeled as random vectors drawn from a Dirichlet distribution \cite{Mauldon:1959, Mosimann:1962} $f(\bp_i|\bx_i;\btheta)=f(\bp_i;\balpha_i)$ conditioned on the input $\bx_i$ and weights $\btheta$. Dirichlet prior networks have inputs $\bx_i$ and outputs concentration parameters $\balpha_i$.

Given the probability simplex as $\mathcal{S} = \left\{(p_1,\dots,p_K): p_i \geq 0, \sum_i p_i=1\right\}$, the Dirichlet distribution is a probability density function on vectors $\bp \in \mathcal{S}$ given by
\begin{equation} \label{eq:Dirichlet}
	f(\bp;\balpha) = \frac{1}{B(\balpha)} \prod_{j=1}^K p_j^{\alpha_j-1}
\end{equation}
where $B(\balpha)=\prod_{j=1}^K \Gamma(\alpha_j)/\Gamma(\alpha_0)$ is the multivariate Beta function. It is characterized by concentration parameters $\balpha = (\alpha_1,\dots,\alpha_K)$, here assumed to be larger than unity. \footnote{The reason for this constraint is that the Dirichlet distribution becomes inverted for $\alpha_j<1$ concentrating in the corners of the simplex and along its boundaries.}

The predictive uncertainty of a classification model trained over this dataset can be expressed as:
\begin{align*}
	&P(y=j|\bx^*,\mathcal{D}) = \int P(y=j|\bx^*,\btheta) p(\btheta|\mathcal{D}) d\btheta \\
		&\quad = \int \int P(y=j|\bp) f(\bp|\bx^*,\btheta) d\bp \cdot p(\btheta|\mathcal{D}) d\btheta \\
		&\quad = \int P(y=j|\bp) f(\bp|\bx^*,\mathcal{D}) d\bp
\end{align*}
The terms above represent data uncertainty, $P(y=j|\bp)$, distribution uncertainty, $f(\bp|\bx^*,\btheta)$, and model uncertainty, $p(\btheta|\mathcal{D})$. The Bayesian hierarchy implies that model uncertainty affects distributional uncertainty, which as a result influences the data uncertainty estimates. In our framework, the additional level of distributional uncertainty is incorporated to control the information spread over the simplex by learning $f(\bp|\bx^*,\btheta)$ in a robust manner during the training procedure. This in turn regularizes the density $f(\bp|\bx^*,\mathcal{D})$ to produce improved predictive uncertainty estimates.

Since the posterior $p(\btheta|\mathcal{D})$ is intractable, approximate variational inference methods may be used in similar spirit to \cite{Blundell:ICML:2015, Gal:ICML:2016} to estimate it. In addition, ensemble approaches are computationally expensive. For clarity in this paper, we assume a point-estimate of the weight parameters is sufficient given a large training set and proper regularization control, which yields $f(\bp|\bx^*,\mathcal{D}) \approx f(\bp|\bx^*,\bar{\btheta})$. This simplifying approximation was also made in recent works \cite{Sensoy:NIPS:2018, Malinin:2019}.

Conventional NNs for classification trained with a cross-entropy loss with a softmax output layer provide a \textit{point estimate} of the predictive class probabilities of each example and do not have a handle on the underlying uncertainty. Cross-entropy training may be interpreted as maximum likelihood estimation which cannot infer predictive distribution variance. This is the prevalent setting for training neural networks for classification, which tends to produce overconfident wrong predictions.

 Concentration parameters may be interpreted as how likely a class is relative to others. In the special case of the all-ones $\balpha$ vector, the distribution becomes uniform over the probability simplex (see Fig. \ref{fig:dirichlet_graphic}(d)). The mean of the proportions is given by $\hat{p}_j = \alpha_j/\alpha_0$, where $\alpha_0=\sum_j \alpha_j$ is the Dirichlet strength.

The Dirichlet distribution is conjugate to the multinomial distribution with posterior parameters updated as $\alpha_j'=\alpha_j+y_j$ for a multinomial sample $\by=(y_1,\dots,y_K)$. For a single sample, $y_j=I_{\{j=c\}}$, where $c$ is the index of the correct class. A Dirichlet neural network's output layer parametrizes the simplex distribution representing the spread of class assignment probabilities. The softmax classification layer is replaced by a softplus activation layer that outputs non-negative continuous values, obtaining
\begin{equation*}
	\balpha = g_\alpha(\bx^*; \bar{\btheta}) + 1
\end{equation*}
that parametrize the density $f(\bp|\bx^*, \bar{\btheta}) = f(\bp; \balpha)$. The posterior distribution $P(y|\bx^*,\bar{\btheta})$ is given by:
\begin{equation*}
	P(y=j|\bx^*;\bar{\btheta}) = \EE_{\bp\sim f(\bp|\bx^*;\bar{\btheta})}[P(y=j|\bp)] = \frac{\alpha_j}{\alpha_0}
\end{equation*}


The concentration parameters determine the shape of the Dirichlet distribution on the probability simplex, as is visualized in Fig. \ref{fig:dirichlet_graphic} for $K=3$. Fig. \ref{fig:dirichlet_graphic}(a) shows a confident prediction characterized by low entropy, (b) shows a more challenging prediction that has higher uncertainty, (c) shows a prediction characterized by high data uncertainty due to class overlap, and (d) shows a flat Dirichlet distribution that arises for an out-of-distribution example.
\begin{figure}[ht]
	\centering
		\includegraphics[width=0.85\textwidth]{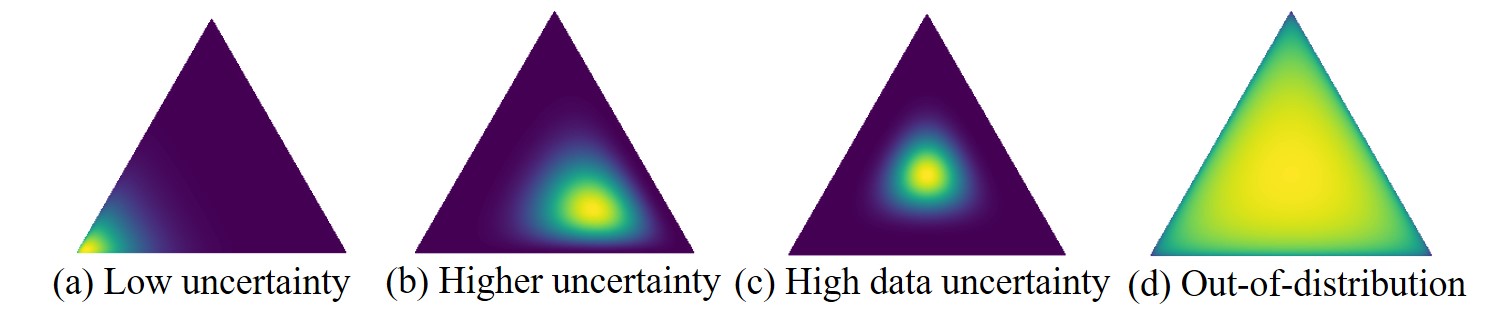}
	\caption{\small Illustration of Dirichlet distribution over categorical class probability distributions (similar to \cite{Malinin:2019}). }
	\label{fig:dirichlet_graphic}
\end{figure}

Predictive entropy measures total uncertainty and may be decomposed into epistemic (or knowledge) uncertainty (arises due to model's difficulty in understanding inputs) and aleatoric (or data) uncertainty (arises due to class-overlap and noise) \cite{Malinin:2019}, given by:
\begin{equation*}
	H(P(y|\bx^*, \bar{\btheta})) = H(\EE_{\bp\sim f(\cdot|\bx^*;\bar{\btheta})}[P(y|\bp)]) = -\sum_j \frac{\alpha_j}{\alpha_0} \log \frac{\alpha_j}{\alpha_0}
\end{equation*}
The mutual information between the labels $y$ and the class probability vector $\bp$, $I(y,\bp|\bx^*;\bar{\btheta})$, captures epistemic uncertainty, and can be calculated by subtracting the expected data uncertainty from the total uncertainty:
\begin{align*}
	I&(y,\bp|\bx^*;\bar{\btheta}) \\
		&= H(\EE_{\bp\sim f(\cdot|\bx^*;\bar{\btheta})}[P(y|\bp)]) - \EE_{\bp\sim f(\cdot|\bx^*,\bar{\btheta})}[H(P(y|\bp))] \\
		&= -\sum_j \frac{\alpha_j}{\alpha_0} \left( \log \frac{\alpha_j}{\alpha_0} - \psi(\alpha_j+1) + \psi(\alpha_0+1) \right)
\end{align*}
where $\psi(\cdot)$ denotes the digamma function. This metric explicitly captures the spread due to distributional uncertainty and is particularly useful for detection of out-of-distribution and adversarial examples. A variation of it was used in the context of active learning \cite{Houlsby:2011}. 

\section{Learning Framework}

\subsection{Classification Loss}
Available are one-hot encoded labels $\by_i$ of examples $\bx_i$ with correct class $c_i$. Treating the Dirichlet distribution $f_{\balpha_i}(\bp_i)$ as a prior on the multinomial likelihood function $\prod_k p_{ik}^{y_{ik}}$, one can minimize the negative log-marginal likelihood:
\begin{align*}
	-\log\left( \EE_{\bp_i\sim f(\cdot;\balpha_i)}\left[ \prod_k p_{ik}^{y_{ik}} \right] \right) = -\log\left( \frac{\alpha_{i,c_i}}{\sum_{j}\alpha_{ij}} \right)
\end{align*}
or the Bayes risk of the cross-entropy loss:
\begin{equation*}
	\EE_{\bp_i\sim f(\cdot;\balpha_i)} \left[ -\sum_k y_{ik} \log p_{ik} \right] = -\left( \psi(\alpha_{i,c_i}) - \psi(\sum_{j}\alpha_{ij}) \right)
\end{equation*}
It was observed in \cite{Sensoy:NIPS:2018} that these loss functions generate excessively high belief masses for classes hurting quantification of uncertainty and are less stable than minimizing the sum of squares of prediction errors instead. This can be attributed to the nature of these loss functions encouraging the maximization of correct class likelihoods.

Unlike conventional cross-entropy training that only seeks to maximize the correct class likelihood, we propose a distance-based objective that minimizes the expected prediction error capturing errors across all classes simultaneously by learning the appropriate Dirichlet concentration parameters that govern the spread of class probability vectors. We propose to minimize the Bayes risk of the prediction error in $L_\infty$ space, which we approximate by relaxing the norm to the $L_p$ space and further use Jensen's inequality as
\begin{align}
	\EE_{\bp_i\sim f(\cdot;\balpha_i)} \nn \by_i-\bp_i \nn_\infty
	    &\leq \EE_{\bp_i\sim f(\cdot;\balpha_i)} \nn \by_i-\bp_i \nn_p \nonumber \\
	    &\leq \left( \EE_{\bp_i\sim f(\cdot;\balpha_i)}[ \nn \by_i-\bp_i \nn_p^p ] \right)^{1/p} \nonumber \\
		&= \left( \EE[(1-p_{i,c_i})^p] + \sum_{j\neq c_i} \EE[p_{ij}^p] \right)^{1/p} =: \mathcal{F}_i(\btheta) \label{eq:Lp_loss}
\end{align}
where we made use of the norm inequality $\nn \be_i \nn_\infty \leq \nn \be_i \nn_p$. The larger $p$, the tighter the $L_p$ norm approximates the max-norm of the prediction error. As the expectation of the max-norm is difficult to directly optimize, Jensen's inequality yields a tractable upper bound that encompasses higher-order moments of the Dirichlet experiment generated by the NN as opposed to just the bias and variance for the $L_2$ case. 
In practice, $p$ is chosen to strike a balance between the correct prediction confidence and uncertainties of errors/out-of-distribution queries.

Our loss improves upon the mean-square-error loss $\EE_{\bp_i\sim f(\cdot;\balpha_i)}[\nn \by_i-\bp_i \nn_2^2]$ proposed in \cite{Sensoy:NIPS:2018} because the $L_\infty$ norm minimizes the cost of the highest prediction error among the classes, while the $L_2$ norm minimizes the sum-of-squares easily affected by outlier scores, $\nn \be_i \nn_\infty \leq \nn \be_i \nn_p \leq \nn \be_i \nn_2$ for $p>2$, and as a result, when errors are made, the uncertainty is expected to be higher as we mitigate the effect of favoring one class more than others. It also improves upon the proposed KL-loss in \cite{Malinin:2019} $D_{KL}( f(\cdot;\alpha_i) \parallel f(\cdot;(\beta+1)\by_i + (1-\by_i)))$ for some arbitrary target parameter $\beta$ as our loss does not require specifying a target true class concentration parameter, and instead tries to fit the best Dirichlet prior to each training example (since one cannot expect all examples to yield highly-concentrated Dirichlet prior distributions); perhaps more importantly we do not rely on access to OOD data at training time.

To calculate each term in $\mathcal{F}_i(\btheta)$, we note $1-p_{i,c_i}$ has a distribution $\text{Beta}(\alpha_{i,0}-\alpha_{i,c_i},\alpha_{i,c_i})$ due to mirror symmetry, and $p_{ij}$ has distribution $\text{Beta}(\alpha_{i,j},\alpha_{i,0}-\alpha_{i,j})$. Marginals of the Dirichlet distribution are Beta random variables, $p_j \sim \text{Beta}(\alpha_j,\alpha_0-\alpha_j)$ with support on $[0,1]$. The $q$-th moment of the Beta distribution $\text{Beta}(a,b)$ is given by
\begin{equation} \label{eq:moment_beta}
	\EE[p^q] = \int_0^1 p^q \frac{p^{a-1} (1-p)^{b-1}}{B_u(a,b)} dp = \frac{B_u(a+q,b)}{B_u(a,b)}
\end{equation}
where $B_u(a,b) = \Gamma(a)\Gamma(b)/\Gamma(a+b)$ is the univariate Beta function. Using the moment expression (\ref{eq:moment_beta}):
\begin{align*}
	\mathcal{F}_i(\btheta) &= \Bigg( \frac{B_u(\alpha_{i,0}-\alpha_{i,c_i}+p,\alpha_{i,c_i})}{B_u(\alpha_{i,0}-\alpha_{i,c_i},\alpha_{i,c_i})} + \sum_{j\neq c_i} \frac{B_u(\alpha_{i,j}+p,\alpha_{i,0}-\alpha_{i,j})}{B_u(\alpha_{i,j},\alpha_{i,0}-\alpha_{i,j})} \Bigg)^{\frac{1}{p}} \\
		&= \left( \frac{\Gamma(\alpha_0)}{\Gamma(\alpha_0+p)} \right)^{\frac{1}{p}} \left( \frac{\Gamma\left(\sum\limits_{k\neq c} \alpha_k+p\right)}{\Gamma\left(\sum\limits_{k\neq c} \alpha_k \right)} + \sum_{k\neq c} \frac{\Gamma(\alpha_k+p)}{\Gamma(\alpha_k)} \right)^{\frac{1}{p}}
\end{align*}

The following theorem shows that the loss function $\mathcal{F}_i$ has the correct behavior as the information flow increases towards the correct class which is consistent when an image sample of that class is observed in a Bayesian Dirichlet experiment and hyperparameters are incremented (see Section \ref{sec:dpn}).
\begin{theorem} \label{thm:loss_decay_correctclass}
	For a given sample $\bx_i$ with correct label $c$, the loss function $\mathcal{F}_i$ is strictly convex and decreases as $\alpha_c$ increases (and increases when $\alpha_c$ decreases).
\end{theorem}
Theorem \ref{thm:loss_decay_correctclass} shows that our objective function encourages the learned distribution of probability vectors to concentrate towards the correct class, consistent with Dirichlet sampling experiments. While increasing information flow towards the correct class reduces the loss, it is also important for the loss to capture elements of incorrect classes. It is expected that increasing information flow towards incorrect classes increases uncertainty.
\begin{theorem} \label{thm:loss_attn}
	For a given sample $\bx_i$ with correct label $c$, the loss function $\mathcal{F}_i$ is increasing in $\alpha_j$ for any $j\neq c$ as $\alpha_j$ grows.
\end{theorem}
Theorem \ref{thm:loss_attn} implies that through minimizing the loss function the model avoids assigning high concentration parameters to incorrect classes as the model cannot explain observations that are assigned incorrect outcomes. The proofs are included in the Appendix.

\subsection{Information Regularization Loss}
The classification loss can discover interesting patterns in the data to achieve high classification accuracy. However, the network may learn that certain patterns lead to strong information flow towards incorrect classes, e.g., a common pattern of one correct class might contribute to a large $\alpha_j$ associated with an incorrect class. While for accuracy this might not be an issue as long as $\alpha_c$ is larger than the incorrect $\alpha_j$, it does affect its predictive uncertainty. Thus, it is of interest to minimize the contributions of concentration parameters associated with incorrect outcomes.

Given the auxiliary vector $\tilde{\balpha}_i = (1-\by_i) \odot \balpha_i + \by_i$ formed by nulling out the correct class concentration parameter $\alpha_{c_i}$, we minimize the following distance function that aligns the concentration parameter vector $\tilde{\balpha}$ towards unity:
\begin{align}
	\mathcal{R}_i &\defequal \frac{1}{2} (\tilde{\balpha}_i-\b1)^T \diag(J(\tilde{\balpha}_i)) (\tilde{\balpha}_i-\b1) \nonumber \\
		&= \frac{1}{2} \sum_{j \neq c_i} (\alpha_{ij}-1)^2 (\psi^{(1)}(\alpha_{ij})-\psi^{(1)}(\tilde{\alpha}_{i0})) \label{eq:inf_reg}
\end{align}
where $\psi^{(1)}(z)=\frac{d}{dz}\psi(z)$ is the polygamma function of order $1$, and $J(\tilde{\balpha})$ denotes the Fisher information matrix $\EE[\nabla \log f(\bp; \tilde{\balpha}) \nabla \log f(\bp; \tilde{\balpha})^T]=-\EE[\nabla^2 \log f(\bp; \tilde{\balpha})]$. We remark that (\ref{eq:inf_reg}) is not a quadratic function in $\alpha_{ij}$ due to the nonlinearity of the polygamma functions and the fact that terms are tied together through the constraint $\tilde{\alpha}_{i0}=1+\sum_{j\neq c} \alpha_{ij}$. This regularization is related to a local approximation of the R\'enyi information divergence \cite{Renyi:1961, Erven:2014} of the Dirichlet distribution $f(\bp; \tilde{\balpha})$ from the uniform Dirichlet $f(\bp; \b1)$ given by
\begin{align*}
		&D_u^R(f(\bp; \tilde{\balpha}) \parallel f(\bp; \b1)) \cong \frac{u}{2} (\tilde{\balpha}-\b1)^T J(\tilde{\balpha}) (\tilde{\balpha}-\b1) \\
		&= \frac{u}{2} \Big[ \sum_{j\neq c}(\alpha_j-1)^2 (\psi^{(1)}(\alpha_j)-\psi^{(1)}(\tilde{\alpha}_0)) \\
		&\quad - \psi^{(1)}(\tilde{\alpha}_0) \sum_{i\neq j,i\neq c,j\neq c} (\alpha_i-1)(\alpha_j-1) \Big]
\end{align*}
in the local regime $\nn \tilde{\balpha}-\b1\nn_2^2= \sum_{j\neq c}(\alpha_j-1)^2 \to 0$. This approximation follows from \cite{Haussler:1997} (p. 2472) after using the second-order Taylor's expansion and substituting the Fisher information matrix $J(\tilde{\balpha}) = \text{diag}(\{\psi^{(1)}(\tilde{\alpha}_i)\}_{i=1}^K) - \psi^{(1)}(\tilde{\alpha}_0) 1_{K\times K}$. The next theorem shows a desirable monotonicity property of the information regularization loss (\ref{eq:inf_reg}).
\begin{theorem} \label{thm:monotonicity_inf_reg}
	The information regularization loss $\mathcal{R}(\alpha)$ given in (\ref{eq:inf_reg}) is increasing in $\alpha_j$ for $j\neq c$.
\end{theorem}
Theorem \ref{thm:loss_attn} and \ref{thm:monotonicity_inf_reg} imply that the strength of concentration parameters associated with misleading outcomes is expected to decrease during training. This preferable behavior of our objective function leads to higher uncertainties for misclassifications as the concentration parameters are all aimed to be minimized instead of allowing one to be much larger than others.

\subsection{Implementation Details}
The total loss to be minimized is:
\begin{equation} \label{eq:reg_loss}
	\mathcal{L}(\btheta) = \frac{1}{N} \sum_{i=1}^N \mathcal{F}_i(\btheta) + \lambda \mathcal{R}_i(\btheta)
\end{equation}
where $\lambda$ is a nonnegative parameter controlling the tradeoff between minimizing the approximate Bayes risk and the information regularization penalty.

Our method modifies the output layer of neural networks and the training loss, therefore maintaining computational efficiency and ease of implementation. Once the network architecture is set and the regularized loss (\ref{eq:reg_loss}) is defined, training is performed using a gradient-based optimizer using minibatches with $\lambda$ increasing using an annealing schedule, e.g., $\lambda_t = \lambda \min\{\frac{t-T_0}{T}, 1\}$ for $t>T_0$ for rate parameter (e.g. $T=60$) and $\lambda_t=0$ for $t\leq T_0$. The gradual annealing allows the network to learn discriminative features for classification first before introducing the information penalty. Training is stopped once the test/validation loss does not improve after $20$ epochs.

The computational complexity of computing a gradient step is $O(N_w)$ for each epoch and training example, where $N_w$ are the number of network parameters as it's based on the runtime complexity of backpropagation. The total train time complexity assuming a batch size of $B$ and $N_e$ epochs is $O(N_e B N_w)$.

\section{Experimental Results} \label{sec:experimental_results}
All experiments are implemented in Tensorflow \cite{Abadi:2016:TSL:3026877.3026899} and the Adam \cite{Kingma:2015} optimizer was used for training. As recent prior works \cite{Sensoy:NIPS:2018, Malinin:2019} have shown Dirichlet NNs outperforming BNNs on several benchmark image datasets, we mainly focus on comparing our method with these Dirichlet NNs trained with different loss functions. Comparisons are made with the following methods: (a) L2 corresponds to deterministic neural network with softmax output and weight decay, (b) Dropout is the uncertainty estimation method of \cite{Gal:ICML:2016}, (c) EDL is the evidential approach of \cite{Sensoy:NIPS:2018}, (d) RKLPN is the reverse KL divergence-based prior network method of \cite{Malinin:2019} with no OOD regularization, and (e) IAD is our proposed technique.

\subsection{Fashion-MNIST Dataset}
The LeNet CNN architecture with $20$ and $50$ filters of size $5 \times 5$ is used for the Fashion-MNIST dataset \cite{FashionMNIST} with $500$ hidden units at the dense layer. The train/test set contains $60,000$/$10,000$ examples. The results were generated with $\lambda=0.5, p=4$. Table \ref{acc-table-fashmnist} shows the test accuracy on Fashion-MNIST for these methods; IAD is shown to be competitive assigning low uncertainty to correct predictions and high uncertainty to errors. In general, a small accuracy loss is expected as the NN is trained so that data examples near the decision boundary (likely errors) lie in a high-uncertainty region that might affect predictions of nearby data; this can be mitigated by adjusting $\lambda$ or $p$. However, our results show that accuracy loss is not significant and OOD/adversarial uncertainty quantification improves upon prior methods while maintaining low uncertainty on correct predictions.
\begin{table}[t]
\caption{ \small Fashion-MNIST Dataset: Test accuracy ($\%$), median predictive entropy for correct and misclassified examples for various deep learning methods.}
\label{acc-table-fashmnist}
\vskip 0.15in
\begin{center}
\begin{small}
\begin{sc}
\begin{tabular}{ p{40pt} p{40pt} p{50pt} p{50pt} }
\toprule
Method & Accuracy & Median Entropy-Successes & Median Entropy-Errors \\
\midrule
L2      & 91.4 & 0.01 & 0.67 \\
Dropout & 91.4 & 0.17 & 0.93 \\
EDL     & 91.6 & 0.58 & 1.50 \\
RKLPN   & 92.5 & 0.48 & 1.19 \\
IAD     & 90.6 & 0.20 & 2.30 \\
\bottomrule
\end{tabular}
\end{sc}
\end{small}
\end{center}
\vskip -0.1in
\end{table}

To measure within-distribution uncertainty, Fig. \ref{fig:fashmnist-id-boxplot} shows boxplots of predictive distribution entropy for correct and misclassified examples across competing methods. The overconfidence of conventional L2 NNs is evident since the distribution mass of correct and wrong predictions is concentrated on lower uncertainties. The Dirichlet-based methods, EDL and RKLPN, tend to sacrifice correct class confidence for providing higher uncertainties on misclassified examples. IAD offers a drastic improvement over all methods with $63\%$ of the misclassified samples falling within $95\%$ of the max-entropy ($\log 10\approx 2.3$), as opposed to $3\%$ and $4\%$ of the misclassified samples of the RKLPN and EDL methods respectively.
\begin{figure}[ht]
	\centering
		\includegraphics[width=0.80\textwidth]{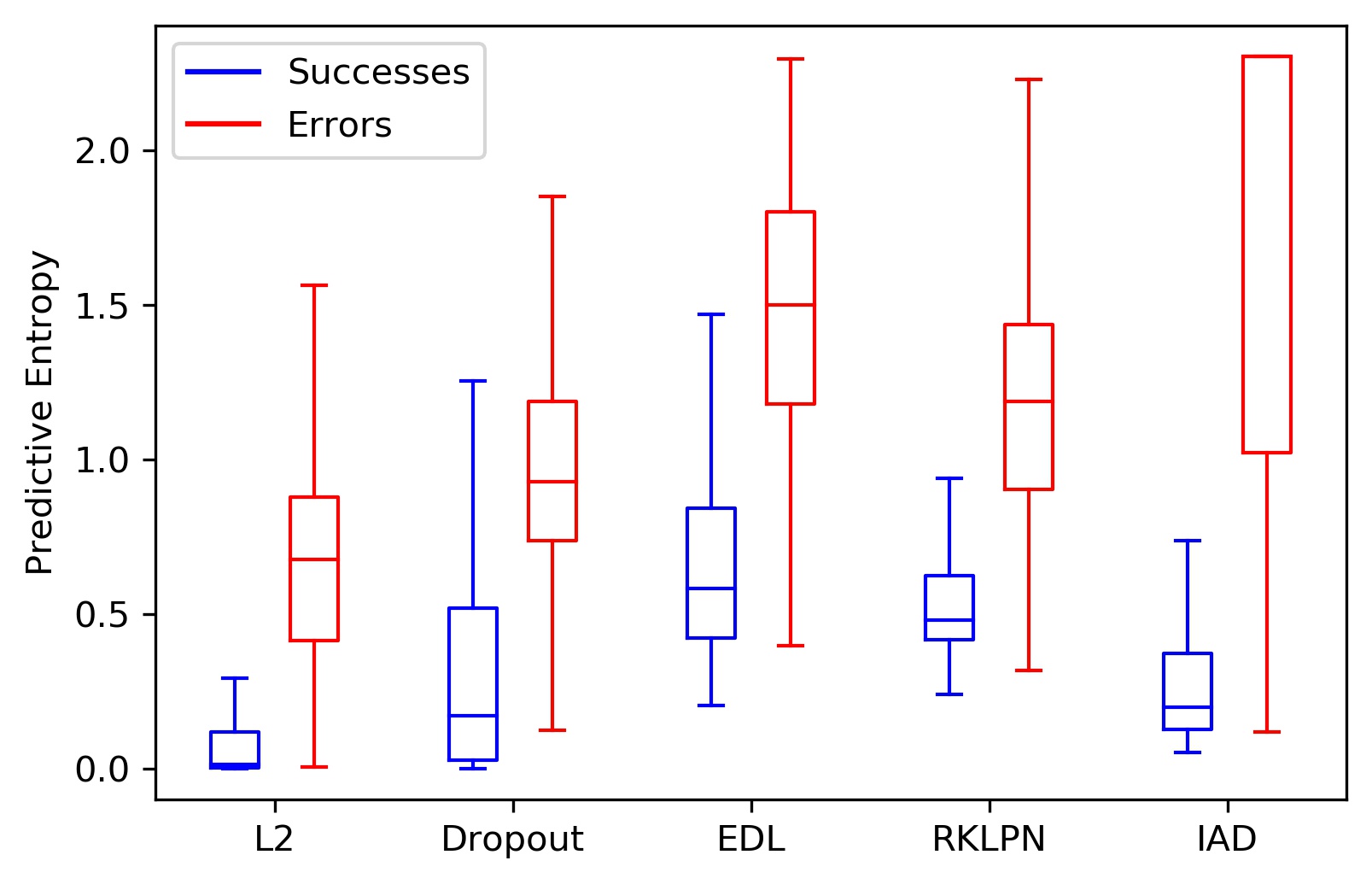}
	\caption{\small Boxplots of predictive distribution entropy for successes and errors on Fashion-MNIST dataset.}
	\label{fig:fashmnist-id-boxplot}
\end{figure}

To evaluate out-of-distribution uncertainty quantification, the trained model on Fashion-MNIST is tested with image data from different datasets. Specifically, IAD is tested on notMNIST \cite{notMNIST} which contains only English letters, and OmniGlot \cite{OmniGlot} which contains characters from multiple alphabets, serving as out-of-distribution data. The uncertainty is expected to be high for all such images as they do not fit into any trained category. Figure \ref{fig:fashmnist-ood-boxplot} shows boxplots of the predictive entropy and mutual information; and it's more desirable to have these metrics higher. IAD is much more tightly concentrated towards higher entropy values; for notMNIST/OmniGlot, an impressive $60\%$/$72\%$ of images have entropy larger than $95\%$ of the max-entropy, while EDL and PN have $5\%$/$10\%$ and $9\%$/$14\%$ approximately.
\begin{figure}[ht]
	\centering
		\includegraphics[width=0.95\textwidth]{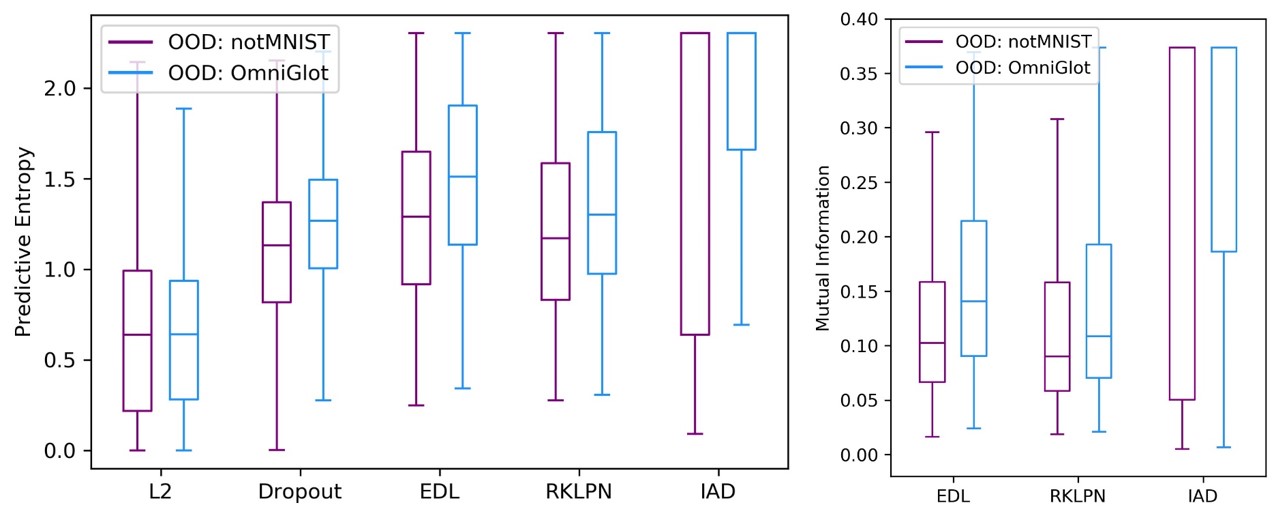}
	\caption{\small Boxplots of predictive distribution entropy (left) and mutual information (right) on out-of-distribution data (notMNIST, OmniGlot) when network is trained on Fashion-MNIST dataset.}
	\label{fig:fashmnist-ood-boxplot}
\end{figure}

Adversarial uncertainty quantification on Fashion-MNIST was also evaluated. Figure \ref{fig:fashmnist-adv-fgsm} shows the adversarial performance when each model is evaluated using adversarial examples generated with the untargeted Fast Gradient Sign method (FGSM) \cite{Goodfellow:2014} for different noise values $\epsilon$, i.e., $\bx_{adv} = \bx + \epsilon \text{sgn}(\nabla_{\bx} \mathcal{F}(\bx,y,\btheta))$. We observe that IAD achieves higher predictive uncertainty on adversarial examples as $\epsilon$ increases than other methods while achieving lower uncertainty for $\epsilon=0$ due to the higher confidence of correct predictions. The quantile spread is shown in Fig. \ref{fig:fashmnist-fgsm-boxplot} for a given noise value. Interestingly, a large entropy is also assigned to misclassified samples as Fig. \ref{fig:fashmnist-id-boxplot} shows.
\begin{figure*}[tp]
	\centering
 \includegraphics[width=0.99\textwidth]{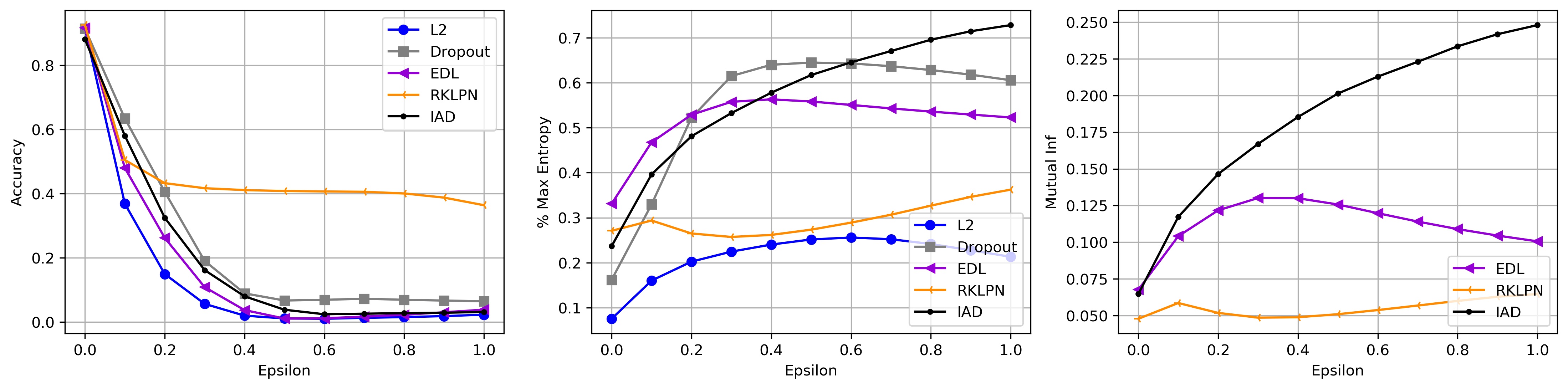}
	\caption{\small Test accuracy (left), predictive entropy (middle) and mutual information (right) for FGSM adversarial examples as a function of adversarial noise $\epsilon$ on Fashion-MNIST dataset. The metrics are averaged over the test set here. }
	\label{fig:fashmnist-adv-fgsm}
\end{figure*}
\begin{figure}[ht]
	\centering
		\includegraphics[width=0.95\textwidth]{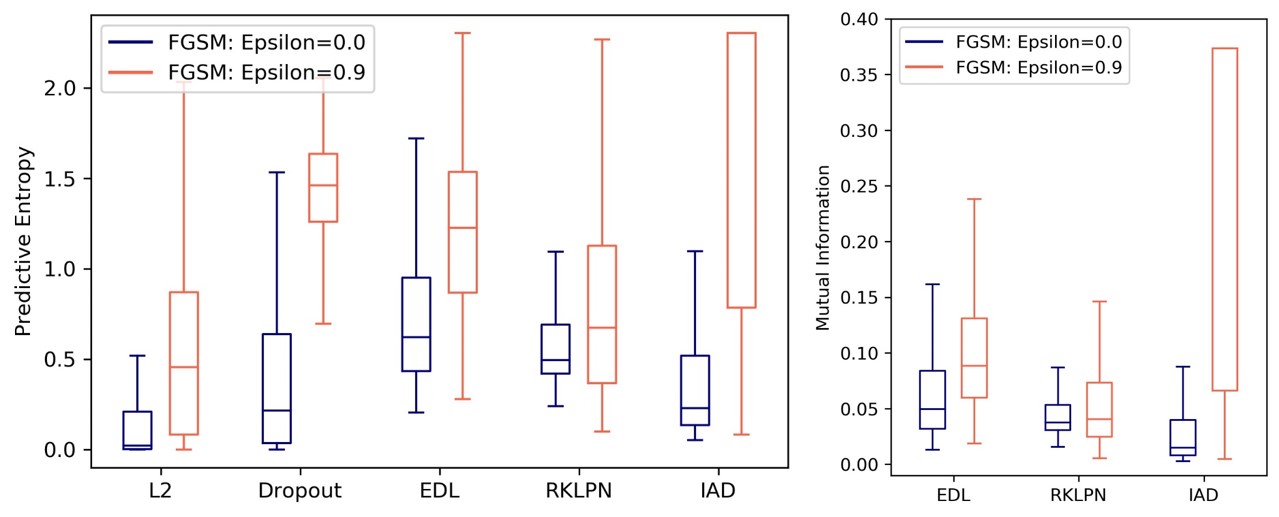}
	\caption{\small Boxplots of predictive distribution entropy (left) and mutual information (right) for clean examples ($\epsilon=0$) and untargeted FGSM perturbations ($\epsilon=0.9$) when network is trained on Fashion-MNIST dataset.}
	\label{fig:fashmnist-fgsm-boxplot}
\end{figure}

The test accuracy as a function of epochs is shown in Fig. \ref{fig:fashmnist-learningcurve} during the training process. While IAD training takes a longer time to converge, it achieves significantly higher predictive uncertainty for OOD and adversarial examples.
\begin{figure}[ht]
	\centering
		\includegraphics[width=0.75\textwidth]{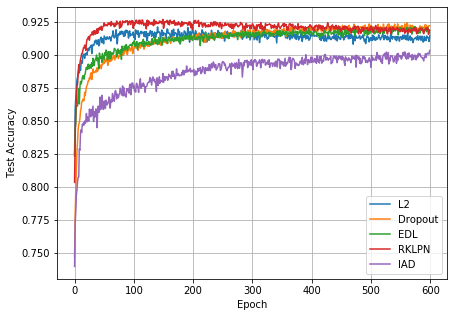}
	\caption{\small Test accuracy as a function of epochs for learning rate $1e-3$ for networks trained on Fashion-MNIST dataset.}
	\label{fig:fashmnist-learningcurve}
\end{figure}

\subsection{CIFAR-10 Dataset}
A VGG-based CNN architecture, consisting of three filter blocks with $64,128,256$ filters, respectively with filter sizes $3 \times 3$, was used for the CIFAR-10 dataset \cite{CIFAR} with $256$ hidden units at the dense layer. The train/test set is made up of $60,000$/$10,000$ examples. Regularization parameter $\lambda=0.3$ was adopted with $p=4$. Data augmentation, dropout and batch-normalization was used for all methods to mitigate overfitting. Table \ref{acc-table-cifar-10} shows the test accuracy on CIFAR-10 for these methods; IAD is shown to be competitive assigning low uncertainty to correct predictions and high uncertainty to errors.
\begin{table}[t]
\caption{\small CIFAR-10 Dataset: Test accuracy ($\%$), median predictive entropy for correct and misclassified examples for various deep learning methods.}
\label{acc-table-cifar-10}
\vskip 0.15in
\begin{center}
\begin{small}
\begin{sc}
\begin{tabular}{ p{40pt} p{40pt} p{50pt} p{50pt} }
\toprule
Method & Accuracy & Median Entropy-Successes & Median Entropy-Errors \\
\midrule
L2      & 85.2 & 0.02 & 0.84 \\
Dropout & 86.7 & 0.14 & 1.10 \\
EDL     & 87.8 & 0.55 & 1.26 \\
RKLPN   & 85.1 & 0.41 & 1.16 \\
IAD     & 85.6 & 0.29 & 1.51 \\
\bottomrule
\end{tabular}
\end{sc}
\end{small}
\end{center}
\vskip -0.1in
\end{table}

Within-distribution uncertainty quantification is evaluated in Fig. \ref{fig:cifar10-id-boxplot} which shows boxplots of predictive distribution entropy for correct and misclassified examples. Similar to the previous set of results, conventional L2 NNs yield overconfident predictions and EDL and RKLPN sacrifice correct class confidence for providing higher uncertainties on misclassified examples. IAD offers an improvement over all methods as the tail of the distribution of predictive entropies associated with misclassified examples is more heavily concentrated on higher values, while maintaining an improved correct prediction confidence over other Dirichlet neural networks.
\begin{figure}[ht]
	\centering
		\includegraphics[width=0.80\textwidth]{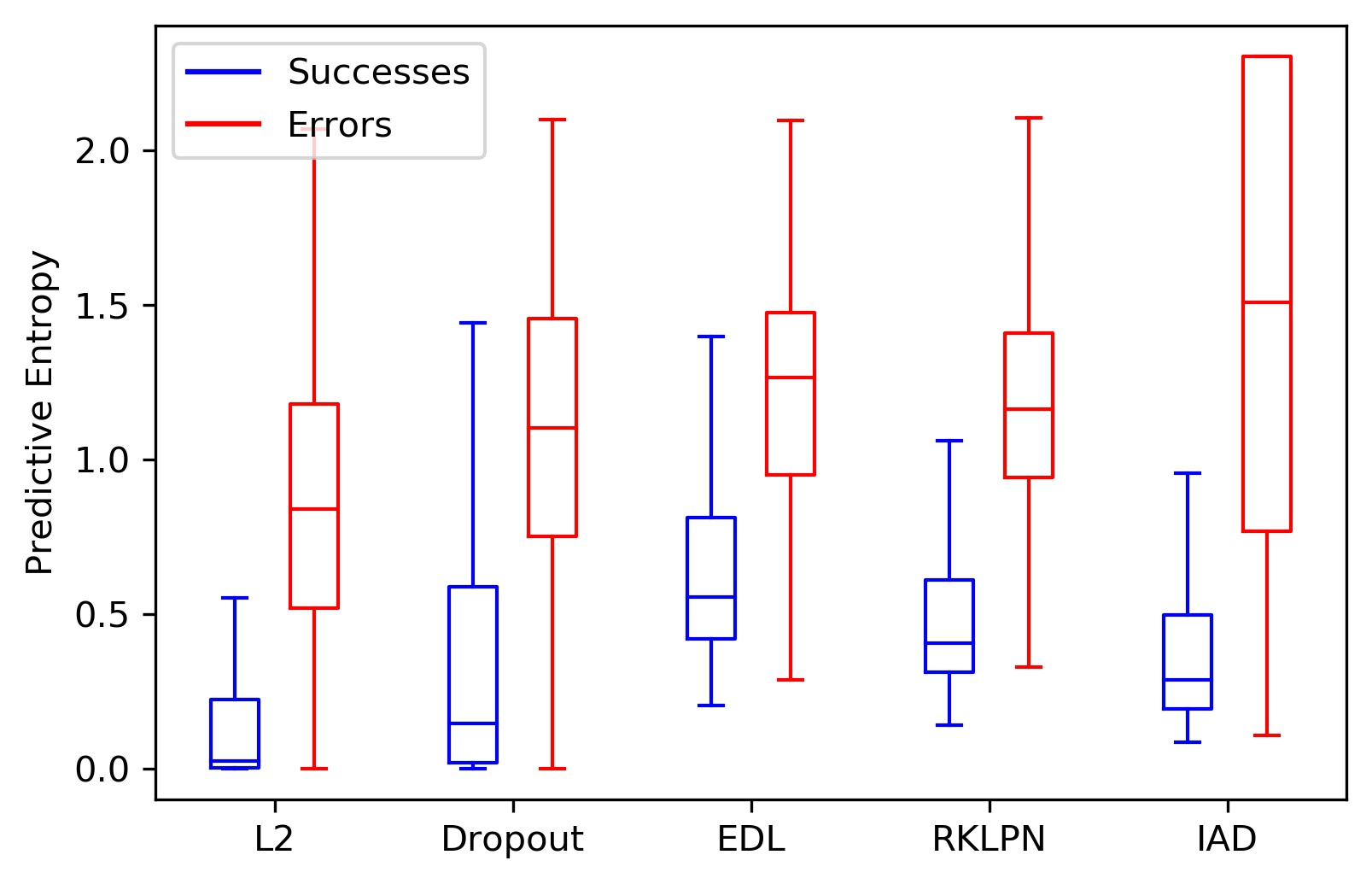}
	\caption{\small Boxplots of predictive distribution entropy for successes and errors on CIFAR-10 dataset.}
	\label{fig:cifar10-id-boxplot}
\end{figure}

For out-of-distribution testing, IAD is tested on SUN \cite{SUN} which contains various environmental scene and places images, and SVHN \cite{SVHN} which contains street-view house numbers. High uncertainty is expected for all such images as they do not fit into any trained category. Figure \ref{fig:cifar10-ood-boxplot} shows the spread of predictive entropy and mutual information using boxplots. IAD improves upon competing methods as it concentrates more towards higher uncertainty metrics.
\begin{figure}[ht]
	\centering
		\includegraphics[width=0.95\textwidth]{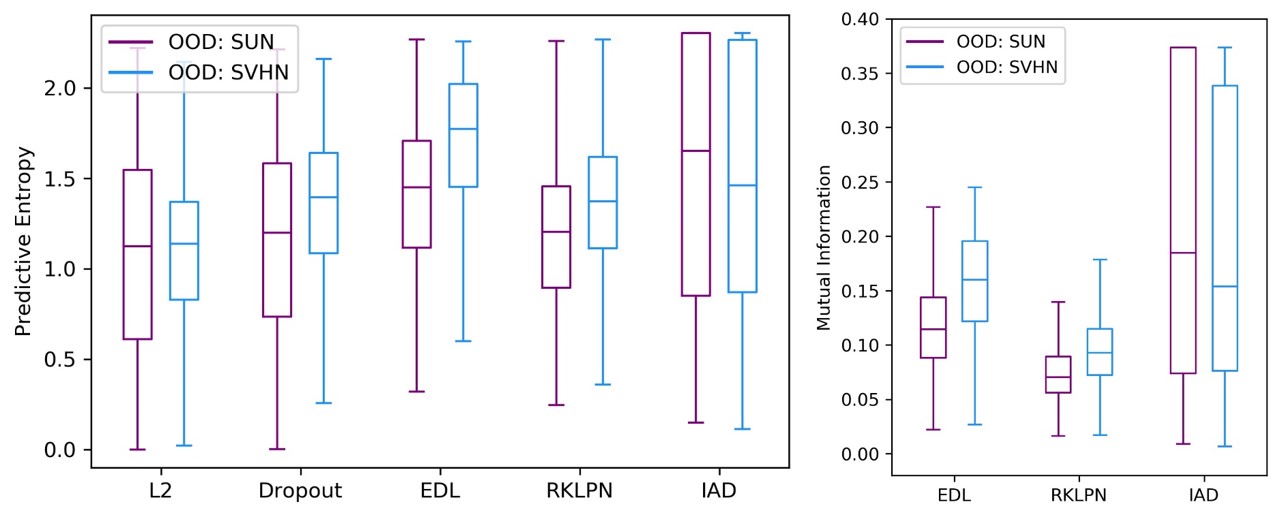}
	\caption{\small Boxplots of predictive distribution entropy (left) and mutual information (right) on out-of-distribution data (SUN, SVHN) when network is trained on CIFAR-10 dataset.}
	\label{fig:cifar10-ood-boxplot}
\end{figure}

The adversarial performance for CIFAR-10 is shown in Fig. \ref{fig:cifar10-adv-fgsm} under FGSM adversarial attacks as a function of noise $\epsilon$. It is observed that IAD starts at low predictive entropy/mutual information and quickly increases its uncertainty as more adversarial noise is added. The spread of the predictive entropy and mutual information distributions are shown in Fig. \ref{fig:cifar10-fgsm-boxplot}. We note IAD offers a significant improvement over other methods in the mutual information metric specifically.
\begin{figure*}[tp]
	\centering
 \includegraphics[width=0.99\textwidth]{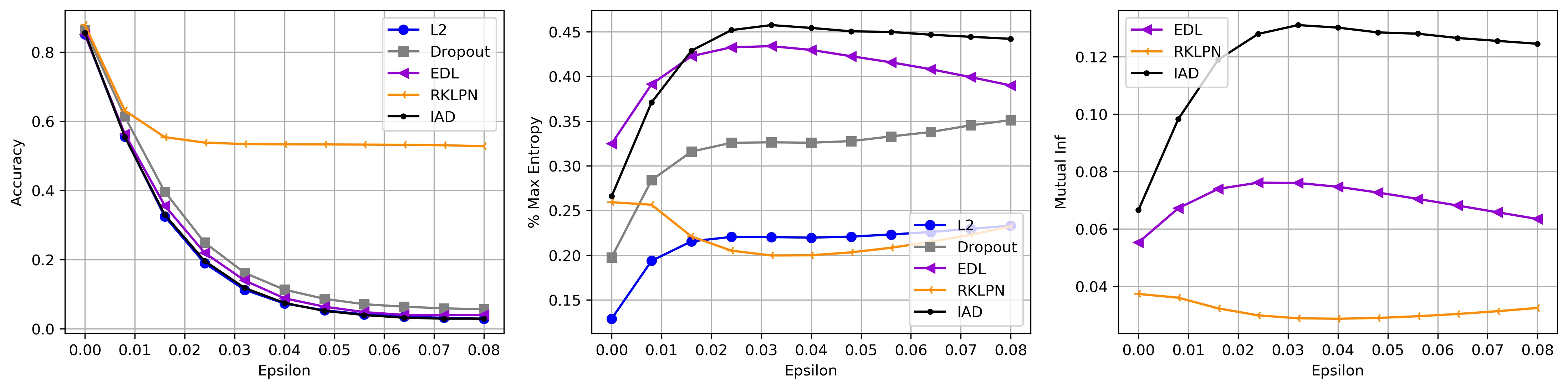}
	\caption{\small Test accuracy (left), predictive entropy (middle) and mutual information (right) for FGSM adversarial examples as a function of adversarial noise $\epsilon$ on CIFAR-10 dataset.}
 \label{fig:cifar10-adv-fgsm}
\end{figure*}
\begin{figure}[ht]
	\centering
		\includegraphics[width=0.95\textwidth]{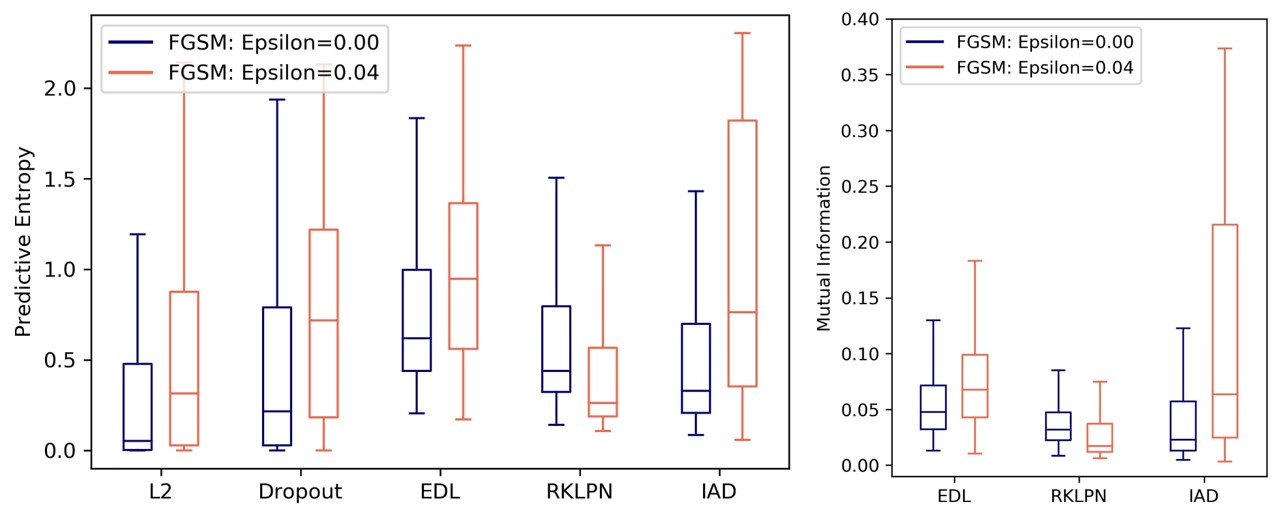}
	\caption{\small Boxplots of predictive distribution entropy (left) and mutual information (right) for clean examples ($\epsilon=0$) and untargeted FGSM perturbations ($\epsilon=0.04$) when network is trained on CIFAR-10 dataset.}
	\label{fig:cifar10-fgsm-boxplot}
\end{figure}

\section{Related Work}
The authors in \cite{Sensoy:NIPS:2018} propose a mean-square error loss, $\EE_{\bp_i\sim f(\cdot;\balpha_i)}[\nn \by_i-\bp_i \nn_2^2] = \sum_k (y_{ik}-\EE[p_{ik}])^2 + \Var(p_{ik})$, and provide a limited analysis by showing the squared-bias decreasing in the true class concentration parameter, $\alpha_{i,c_i}$, and it decreases when the largest incorrect concentration parameter, $\alpha_{ij},j\neq c_i$, decays. However, these properties were not shown for the aggregate loss function or the regularizer, and the behavior of the loss is not studied in terms of all concentration parameters.

The work \cite{Malinin:2019} proposed the KL loss given by $D_{KL}( f(\cdot;\alpha_i) \parallel f(\cdot;(\beta+1)\by_i + (1-\by_i)))$ for some arbitrary target parameter $\beta\gg 1$. The authors do not provide an analysis that related the Dirichlet concentration parameters with their loss, and further require OOD data for learning what is anomalous using an auxiliary loss biasing the predictive uncertainty of the trained model which is a questionable assumption for most applications.

In contrast to these works, we provide theoretical properties of our proposed aggregate loss function. We analytically show the proposed loss is decreasing as the true class concentration parameter grows, and increasing as incorrect concentration parameters grow, which yields insights into how Dirichlet distributions are shaped on the simplex, while not relying on access to OOD data during training.

Our proposed loss function aims to approximate the maximum prediction error $\EE_{\bp_i\sim f(\cdot;\balpha_i)}[\max_k |y_{ik}-p_{ik}|]$ using an $L_p$ norm relaxation (\ref{eq:Lp_loss}). Our max-norm objective minimizes the cost of the highest prediction error among the classes, while the $L_2$ norm minimizes the aggregate sum-of-squares, which is more prone to being affected by outlier scores, e.g. $\nn \be_i \nn_\infty \leq \nn \be_i \nn_p \leq \nn \be_i \nn_2$ for $p>2$. Experiments in Section \ref{sec:experimental_results} show improvements in predictive uncertainty estimation for failure cases associated with within-distribution queries, and anomalous queries including out-of-distribution and adversarial examples.


\section{Conclusion}
In this work, we presented a new method for training Dirichlet neural networks that are aware of the uncertainty associated with predictions.  Our training objective fits predictive distributions to data using a classification loss that minimizes an approximation to the maximum expected prediction error measured in, and an information regularization loss that penalizes information flow towards incorrect classes. We derived closed-form expressions for our training loss and desirable properties on how improved uncertainty estimation is achieved. Experimental results were shown on image classification tasks, highlighting improvements in predictive uncertainty estimation for within-distribution, out-of-distribution and adversarial queries in comparison to conventional neural networks with weight decay, Bayesian neural networks, and other recent Dirichlet networks trained with different loss functions.

Future work directions include evaluating predictive uncertainty of our proposed network to detect stronger adversarial attacks (e.g. iterative FGSM), optimizing the network for adversarial robustness with the information-aware penalty (e.g., using adversarial training), and considering extensions to more flexible prior distributions on the simplex (e.g. Dirichlet mixtures).

\section*{Appendix}

We make use of the following lemmas in the proofs.
\begin{lemma} \label{lem:psi}
	Consider the digamma function $\psi$. Assuming $x_1>x_2>1$ and $p>0$, the following inequality strictly holds:
	\begin{equation*}
		0< \psi(x_1+p)-\psi(x_2+p) < \psi(x_1)-\psi(x_2)
	\end{equation*}
	Furthermore, we have $\lim_{x\to\infty} {\psi(x+p)-\psi(x)} = 0$.
\end{lemma}
\begin{proof}
	Since $x_1>x_2>1$, we can write $x_1=s_1+1$ and $x_2=s_2+1$ for some $s_1>s_2$. Upon substitution of the Gauss integral representation $\psi(z+1) = -\gamma + \int_0^1 \left( \frac{1-t^z}{1-t} \right) dt$ (here $\gamma$ is the Euler-Mascheroni constant), we have:
\begin{equation*}
	\psi(x_1)-\psi(x_2) = \int_0^1 \left( \frac{t^{s_2}-t^{s_1}}{1-t} \right) dt
\end{equation*}
which is strictly positive since the integrand is positive for $t\in (0,1)$. Using the integral representation again, the inequality $\psi(x_1+p)-\psi(x_2+p) < \psi(x_1)-\psi(x_2)$ is equivalent to:
\begin{equation*}
	\int_0^1 \left( \frac{(1-t^p)(t^{s_2}-t^{s_1})}{1-t} \right) > 0
\end{equation*}
which holds since the integrand is positive due to $t^p<1$ an $t^{s_1}<t^{s_2}$. The limit of $\psi(x+p)-\psi(x)$ follows from the asymptotic expansion $\psi(x)=\log(x)-\frac{1}{2x}+O\left(\frac{1}{x^2}\right)$, which yields $\psi(x+p)-\psi(x)\sim \log(1+p/x) - \frac{1}{2(x+p)}+\frac{1}{2x} \to 0$ as $x\to\infty$. This concludes the proof.
\end{proof}

\begin{lemma} \label{lem:psi1}
	Consider the polygamma function of order 1 $\psi^{(1)}(z)=\frac{d}{dz}\psi(z)$. Assuming $x_1>x_2>1$ and $p>0$, the following inequality strictly holds:
	\begin{equation*}
		\psi^{(1)}(x_1)-\psi^{(1)}(x_2) < \psi^{(1)}(x_1+p)-\psi^{(1)}(x_2+p) < 0
	\end{equation*}
\end{lemma}
\begin{proof}
	Proceeding similarly as in the Proof of Lemma \ref{lem:psi}, we write $x_1=s_1+1$ and $x_2=s_2+1$ for some $s_1>s_2$. Upon substitution of the integral representation $\psi^{(1)}(z+1) = \int_0^1 \left( \frac{t^z}{1-t} \ln\left( \frac{1}{t} \right) \right) dt$, we have:
\begin{equation*}
	\psi^{(1)}(x_1)-\psi^{(1)}(x_2) = \int_0^1 \left( \frac{t^{s_1}-t^{s_2}}{1-t} \ln\left( \frac{1}{t} \right) \right) dt
\end{equation*}
which is strictly negative since the integrand is negative for $t\in (0,1)$. Using the integral representation again, the inequality $\psi^{(1)}(x_1)-\psi^{(1)}(x_2) < \psi^{(1)}(x_1+p)-\psi^{(1)}(x_2+p)$ is equivalent to:
\begin{equation*}
	\int_0^1 \left( \frac{(1-t^p)(t^{s_1}-t^{s_2})}{1-t} \ln\left( \frac{1}{t} \right) \right) < 0
\end{equation*}
which holds true since $\ln(1/t)>0$ for $t\in (0,1)$. This concludes the proof.
\end{proof}

\subsection*{Proof of Theorem \ref{thm:loss_decay_correctclass}}
\begin{proof}
Taking the logarithm of $\mathcal{F}_i$, we have:
\begin{equation*}
	\log \mathcal{F}_i = \frac{1}{p} \log \left( \frac{\Gamma(\alpha_0)}{\Gamma(\alpha_0+p)} \right) + \frac{1}{p} \log \left( \frac{\Gamma(\sum_{k\neq c} \alpha_k + p)}{\Gamma(\sum_{k\neq c} \alpha_k)} + \sum_{j\neq c} \frac{\Gamma(\alpha_j+p)}{\Gamma(\alpha_j)} \right)
\end{equation*}
where the second term is independent of $\alpha_c$. Letting the first term be denoted as $g(\alpha_c) := \frac{1}{p} \log \left( \frac{\Gamma(\alpha_0)}{\Gamma(\alpha_0+p)} \right)$, it suffices to show $f(\alpha_c):=\exp(g(\alpha_c))$ is strictly convex and decreasing in $\alpha_c$. Differentiating $g(\alpha_c)$ twice we obtain:
\begin{align*}
	g'(\alpha_c)  &= \frac{1}{p} \left( \psi(\alpha_0) - \psi(\alpha_0+p) \right) \\
	g''(\alpha_c) &= \frac{1}{p} \left( \psi^{(1)}(\alpha_0) - \psi^{(1)}(\alpha_0+p) \right)
\end{align*}
Lemmas \ref{lem:psi} and \ref{lem:psi1} then yield that $g'(\alpha_c)<0$ and $g''(\alpha_c)>0$ respectively. Differentiating $f(\alpha_c)$ twice, we have:
\begin{align*}
	f'(\alpha_c)  &= e^{g(\alpha_c)} g'(\alpha_c) \\
	f''(\alpha_c) &= e^{g(\alpha_c)} \left( g''(\alpha_c) + (g'(\alpha_c))^2 \right)
\end{align*}
Using the inequalities above and the positivity of $e^{g(\alpha_c)}$, it follows that $f'(\alpha_c)<0$ and $f''(\alpha_c)>0$. Thus, $f(\alpha_c)$ is a strictly convex decreasing function in $\alpha_c$. This concludes the proof.
\end{proof}

\subsection*{Proof of Theorem \ref{thm:loss_attn}}
\begin{proof}
Consider a concentration parameter $\alpha_j$ corresponding to an incorrect class, i.e., $j\neq c$. Define the ratio of Gamma functions as:
\begin{equation*} 
	\mu(\alpha) \defequal \frac{\Gamma(\alpha+p)}{\Gamma(\alpha)}
\end{equation*}
This function is positive, increasing and convex with derivative given by:
\begin{align}
	\mu'(\alpha) &= -\frac{\Gamma(\alpha+p)\Gamma'(\alpha)}{\Gamma(\alpha)^2}  + \frac{\Gamma'(\alpha+p)}{\Gamma(\alpha)} \nonumber \\ 
		&= - \frac{\Gamma(\alpha+p)\psi(\alpha)}{\Gamma(\alpha)} + \frac{\Gamma(\alpha+p)\psi(\alpha+p)}{\Gamma(\alpha)} \nonumber \\
		&= \mu(\alpha) \left( \psi(\alpha+p)-\psi(\alpha) \right) \nonumber \\
		&= \mu(\alpha) \nu(\alpha) \label{eq:dmu}
\end{align}
where we used the relation $\Gamma'(z)=\Gamma(z)\psi(z)$ and defined
\begin{equation*}
	\nu(\alpha) \defequal \psi(\alpha+p)-\psi(\alpha).
\end{equation*}
From Lemma \ref{lem:psi}, it follows that $\nu(\alpha)>0$ which implies $\mu(\alpha)$ is increasing.

Since $(\cdot)^{1/p}$ is a continuous increasing function, it suffices to show the objective $\mathcal{G}=\mathcal{F}_i^p$ is increasing, given by $\mathcal{G}(\alpha_j) = \left( \mu\left(\sum\limits_{l\neq c}\alpha_l\right) + \sum\limits_{l\neq c} \mu(\alpha_l) \right) / \mu(\alpha_0)$. The derivative is then calculated as:
\begin{equation*}
	\mathcal{G}'(\alpha_j) = \frac{\mu'\left(\sum\limits_{l\neq c}\alpha_l\right)+\mu'(\alpha_j)}{\mu(\alpha_0)} - \frac{\mu'(\alpha_0) \cdot  \left[ \mu\left(\sum\limits_{l\neq c}\alpha_l\right) + \sum\limits_{l\neq c} \mu(\alpha_l)  \right]}{\mu(\alpha_0)}
\end{equation*}
The condition $\mathcal{G}'(\alpha_j)>0$ is equivalent to:
\begin{equation*}
	\frac{\mu'\left(\sum\limits_{l\neq c}\alpha_l\right)+\mu'(\alpha_j)}{\mu'(\alpha_0)} > \frac{\mu\left(\sum\limits_{l\neq c}\alpha_l\right) + \sum\limits_{l\neq c} \mu(\alpha_l) }{\mu(\alpha_0)} = \mathcal{G}
\end{equation*}
Upon substituting the expression (\ref{eq:dmu}), this condition becomes:
\begin{equation}
	\mu\left(\sum\limits_{l\neq c}\alpha_l\right) \nu\left( \sum\limits_{l\neq c}\alpha_l \right) + \mu(\alpha_j) \nu(\alpha_j) > \left[\mu\left(\sum\limits_{l\neq c}\alpha_l\right) + \sum\limits_{l\neq c} \mu(\alpha_l) \right] \nu(\alpha_0) \label{eq:ineq}
\end{equation}
From Lemma \ref{lem:psi}, it follows that $\nu\left( \sum\limits_{l\neq c}\alpha_l \right) > \nu(\alpha_0)$ and $\nu(\alpha_j) > \nu(\alpha_0)$. In addition, the functions $\mu\left(\sum\limits_{l\neq c}\alpha_l\right) \nu\left( \sum\limits_{l\neq c}\alpha_l \right)$ and $\mu(\alpha_j) \nu(\alpha_j)$ are both increasing as $\alpha_j$ grows. Using these results and the fact that $\left[ \sum\limits_{l\neq c,j} \mu(\alpha_l) \right] \nu(\alpha_0) \to 0$ as $\alpha_j$ grows (due to Lemma \ref{lem:psi}), it follows that the inequality (\ref{eq:ineq}) holds true for large $\alpha_j$. Thus, we conclude that the loss function is increasing as $\alpha_j$ gets large. The proof is complete.
\end{proof}

An illustration of Theorem 2 is shown in Fig. \ref{fig:thm2_illustration} below. An approximate loss function is also shown due to $\lim_{\alpha\to\infty} \frac{\Gamma(\alpha+p)}{\Gamma(\alpha)\alpha^p} = 1$, from which we obtain the approximation $\mu(\alpha) \sim \alpha^p$. This approximation to the loss behaves similarly. Despite the initial dip, the loss is increasing as $\alpha_j$ increases. We remark that the loss is neither convex nor concave in $\alpha_j$.
\begin{figure}[ht]
	\centering
		\includegraphics[width=0.80\textwidth]{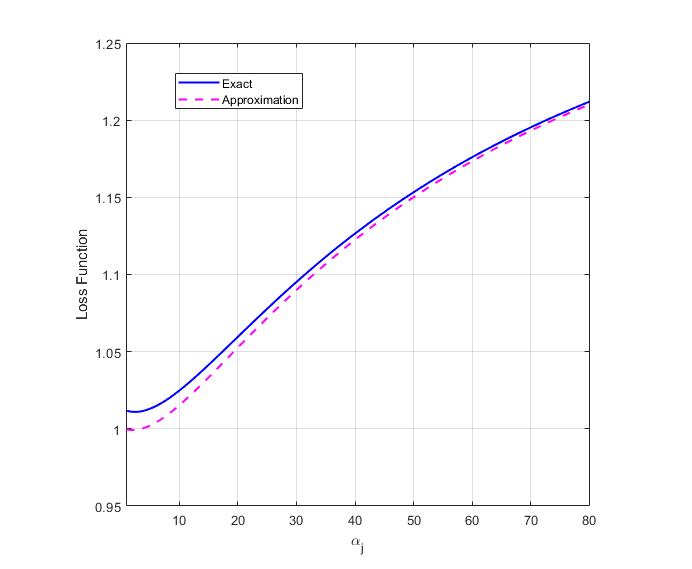}
	\caption{ Illustrative example for Theorem 2. Here, the loss function $\mathcal{F}_i$ is plotted as a function of $\alpha_j$, $j\neq c$. Parameters $p=2$ and a random $\balpha$ vector were used for $K=10$ classes with $\alpha_c$ small relative to other concentration parameters. As Theorem 2 shows, the loss is increasing for large $\alpha_j$. }
	\label{fig:thm2_illustration}
\end{figure}

\subsection*{Proof of Theorem \ref{thm:monotonicity_inf_reg}}
\begin{proof}
Consider $\mathcal{R}(\alpha_k)$ as a function of $\alpha_k$ for some $k\neq c$. Then, it may be decomposed as $\mathcal{R}(\alpha_k) = \mathcal{R}_k(\alpha_k) + \mathcal{R}_{\neq k}(\alpha_k)$ where
\begin{align*}
	\mathcal{R}_k(\alpha_k)        &= \frac{1}{2} (\alpha_k-1)^2 (\psi^{(1)}(\alpha_k)-\psi^{(1)}(\tilde{\alpha}_0)) \\
	\mathcal{R}_{\neq k}(\alpha_k) &= \frac{1}{2} \sum_{j\neq c, j\neq k} (\alpha_j-1)^2 (\psi^{(1)}(\alpha_j)-\psi^{(1)}(\tilde{\alpha}_0))
\end{align*}
The first term is an increasing function since $q(\alpha) = (\alpha-1)^2 (\psi^{(1)}(\alpha)-\psi^{(1)}(\alpha+z))$ is increasing for any $z>1$. The second term is also increasing since 
\begin{align*}
	\frac{\partial \mathcal{R}_{\neq k}(\alpha_k)}{\partial \alpha_k} &= \frac{-\psi^{(2)}(\tilde{\alpha}_0)}{2} \sum_{j\neq c, j\neq k} (\alpha_j-1)^2 \geq 0
\end{align*}
which follows from the integral representation $\psi^{(2)}(x) = -\int_{0}^\infty \frac{t^2 e^{-tx}}{1-e^{-t}} dt \leq 0$.
\end{proof}

\section*{Acknowledgements}
DISTRIBUTION STATEMENT A. Approved for public release. Distribution is unlimited. This material is based upon work supported by the Under Secretary of Defense for Research and Engineering under Air Force Contract No. FA8702-15-D-0001. Any opinions, findings, conclusions or recommendations expressed in this material are those of the author(s) and do not necessarily reflect the views of the Under Secretary of Defense for Research and Engineering.


\bibliography{refs}

\end{document}